\documentclass{ifacconf}

\usepackage{amsmath, amssymb, amsfonts} 
\usepackage{ifthen,tikz} 
\usepackage{pgfplots} 
\pgfplotsset{compat=1.17}
\usepackage{xcolor}
\usepackage{balance}

\usepackage{ulem}
\pgfmathdeclarefunction{erf}{1}{%
  \pgfmathparse{(#1 < 0 ? -1 : 1) * (1 - (0.254829592/(1+0.3275911*abs(#1)) -
  0.284496736/(1+0.3275911*abs(#1))^2 + 1.421413741/(1+0.3275911*abs(#1))^3 -
  1.453152027/(1+0.3275911*abs(#1))^4 + 1.061405429/(1+0.3275911*abs(#1))^5)
  * exp(-#1^2))}%
}

\usepackage{natbib}        

\begin{document}
\begin{frontmatter}

\title{Local Stability and Gaussian Smoothing of Quantized Neural Networks} 

\thanks[footnoteinfo]{\copyright~2026 the authors. This work has been accepted to IFAC for
publication under a Creative Commons Licence CC-BY-NC-ND.}


\author[First]{Sergey Salishev} 
\author[First]{Anton Makarov} 
\author[First]{Oleg Granichin} 

\address[First]{St. Petersburg State University, 
   199034, Saint Petersburg, Russia (e-mail: s.salischev@spbu.ru, a.a.makarov@spbu.ru, o.granichin@spbu.ru)}

\begin{abstract}                
We study Gaussian averaging as a smooth surrogate for quantized neural models. Under bounded local oscillation,
we derive a local dimension-dependent bound on \(|f-g|\), linking Gaussian smoothing to the stability analysis of
discontinuous networks. We compute closed-form Gaussian averages of the rectified linear unit (ReLU) and sign activation functions,
and illustrate the mechanism
on a high-dimensional binary perceptron, where layer-preactivation aggregation under an explicit
quantization-noise surrogate yields the Gaussian envelope used in inference-side smoothing and training-side
smooth surrogate gradients.

\end{abstract}


\end{frontmatter}

\balance
\section{Introduction}

Modern control and estimation systems now deploy low-bit neural networks as observers, controllers, and gain schedulers
on quantized hardware. Their discrete, non-smooth nature blocks standard stability certificates and complicates
gradient-based design, motivating principled smoothing.

Functionally, such networks are often piecewise-affine for continuous inputs, but after digitization of the
inputs and internal signals they behave effectively as piecewise-constant maps; in either description,
quantization, saturation, clipping, and switching nonlinearities induce a discontinuous input-output law. Yet many control and estimation tools
(stability certificates, sensitivity analysis, gradient-based optimization) assume at least Sobolev regularity.
We therefore study Gaussian averaging as a way to construct a continuous analytical proxy for the same discrete
model. The key structural condition is
\textit{bounded local oscillation} (Definition~\ref{def:blo}): a weak quantitative substitute for Lipschitz-type
regularity that controls output variation at a prescribed scale without assuming continuity or differentiability.

Gaussian averaging plays here the role of a mollifier, but the object being smoothed is the hardware-driven
quantized model itself rather than an externally prescribed smooth target. The resulting surrogate is used for
inference-side stability analysis and training-side smooth surrogate gradients, in a spirit related to stochastic
approximation and randomized-control viewpoints
\citep{Spall-1992-SPSA,GranichinVolkovichToledanoKitai2015Randomized}.

The use of matched-variance rounding-noise surrogates is supported by \citet{Salishev-2026-SPSAView}, where the
STE/noise-scale update mechanism is analyzed from an SPSA viewpoint and validated empirically for low-bit
neural-network training. That evidence also has an inference-side reading for smooth image-to-image tasks:
for super-resolution, Salishev et al.\ report that RFDN quantized to W4A4 keeps overlapping final PSNR
confidence intervals across stochastic and deterministic rounding-noise variants. A rigorous control paper
would replace this empirical PSNR evidence by an explicit tail-probability bound, but the observation is
consistent with the probabilistic disturbance interpretation used here: on the tested data distribution,
harmful quantization events are either rare or sufficiently small in magnitude to be absorbed by the
reconstruction task. The formal inference-side smoothing and local error control are provided here by
Theorem~\ref{thm:stability} and the layer-wise CLT mechanism of Theorem~\ref{thm:layer-aggregation}. This
averaging viewpoint also helps interpret empirical smooth surrogates used in practice, including learned step
size quantization (LSQ) \citep{Esser-2019-LSQ} and gradual differentiable noise scale quantization (GDNSQ)
\citep{Salishev-2025-GDNSQ}.

Comparison with explicitly sampled-Gaussian methods is useful. In randomized smoothing
\citep{Cohen-2019-RandomizedSmoothing}, the model is averaged over externally sampled Gaussian perturbations to
obtain a certified smooth proxy. In SmoothHess \citep{Torop-2023-SmoothHess}, Gaussian samples are likewise
drawn explicitly and combined with Stein-type identities to estimate feature interactions in ReLU networks. In
the present paper, by contrast, the Gaussian is not the primitive perturbation law. Theorem~\ref{thm:layer-aggregation}
formalizes a layer-level mechanism: under an explicit i.i.d.\ additive-noise quantization surrogate,
componentwise residuals aggregate inside a preactivation and produce the CLT-Gaussian envelope used to define
the averaged surrogate \(g\).

For control-oriented use, this CLT-Gaussian viewpoint should be read as a compact analytical proxy whose local
failures are rare disturbance events rather than deterministic failure of the surrogate. Modern stochastic,
quantized, and networked-control formulations can allow such events when their probability distribution is
quantified and the closed-loop system attenuates their effect, for example through inertia, bandwidth limitation,
or robustness margins. In this sense, Theorem~\ref{thm:stability} separates the nominal local mismatch
\(\varepsilon_2\) from the Gaussian tail contribution corresponding to excursions outside the trusted region.
A full closed-loop probability-of-bad-event analysis is outside the scope of this short paper. Within this
scope, \(g(\cdot,s)\) can be used when deriving Lyapunov-style stability conditions or sensitivity bounds
\citep{SlotineLi1991AppliedNonlinearControl}, since \(g\) is \(C^\infty\) and its derivatives admit explicit
integral representations.
Here and throughout, ``stability'' means support for Lyapunov-style analysis rather than a full closed-loop
stability theorem.

Specifically, we make the following contributions.

\begin{itemize}
  \item Treating data and parameters uniformly as inputs, we show that Gaussian averaging produces a \(C^\infty\)
        surrogate \(g(x,s)\) with explicit mixed-derivative formulas.
  \item Under boundedness and bounded local oscillation, we derive a local dimension-dependent estimate for
        \(|f(x)-g(x,s)|\) (Theorem~\ref{thm:stability}) in terms of the noise variance, oscillation radius, and
        supremum of \(|f|\).
  \item We compute the Gaussian averages of ReLU and sign functions in closed form and relate them to a high-dimensional
        binary perceptron, where aggregation of quantization residuals in a layer preactivation yields the
        Gaussian envelope used by both inference-side smoothing and training-side surrogate gradients.
\end{itemize}

The paper therefore separates the two roles of noise cleanly: Gaussian convolution gives smoothness, while
approximation fidelity depends on local control of the oscillation of the discrete model.

\section{Model}
We now formalize the setting and define bounded local oscillation, the weak regularity condition used in the
approximation bound. Here and below $|\cdot|$ denotes the Euclidean norm in $\mathbb{R}^d$.
Let \( D \subset \mathbb{R}^d \) be some
bounded open set.

\begin{defn}\label{def:blo}
Let \(E \subset D\). We say that \(f\) has \((\varepsilon_1,\varepsilon_2)\)-\textit{bounded local oscillation on} \(E\) if
\[
|x-y|<\varepsilon_1 \quad \Longrightarrow \quad |f(x)-f(y)|<\varepsilon_2
\]
for every \(x\in E\) and every \(y\in D\). If $E$ is a point, e.g., \(E=\{x\}\), we say that \(f\) has
\((\varepsilon_1,\varepsilon_2)\)-\textit{bounded local oscillation at the point} \(x\).
This is a weak quantitative substitute for Lipschitz-type regularity in a discontinuous setting: it controls
oscillation at a prescribed scale without assuming continuity, differentiability, or a linear modulus of
variation.
\end{defn}

\begin{lem}[Data-supported oscillation]
\label{lem:data-supported-blo}
Let \(S\) be a training set with samples \((x_i,y_i)\), \(i=1,\ldots,N\), and let \(f\) be a trained network.
Assume that \(f\) fits the training labels uniformly with error \(\delta\), that is,
\[
|f(x_i)-y_i|\leq \delta,\qquad i=1,\ldots,N.
\]
Assume also that the labels have empirical local oscillation \(\eta\) at scale \(\varepsilon_1\), namely,
\[
|x_i-x_j|<\varepsilon_1 \quad \Longrightarrow \quad |y_i-y_j|\leq \eta
\]
for all training samples \(x_i,x_j\). Then \(f\) has empirical bounded local oscillation on \(S\) at scale
\(\varepsilon_1\) with constant \(\varepsilon_2=\eta+2\delta\):
\[
|x_i-x_j|<\varepsilon_1 \quad \Longrightarrow \quad |f(x_i)-f(x_j)|\leq \eta+2\delta.
\]
\end{lem}

\begin{pf}
By the triangle inequality,
\[
\begin{aligned}
|f(x_i)-f(x_j)|
&\leq |f(x_i)-y_i|+|y_i-y_j|+|y_j-f(x_j)|\\
&\leq \delta+\eta+\delta=\eta+2\delta.
\end{aligned}
\]
For the strict convention in Definition~\ref{def:blo}, replace \(\eta+2\delta\) by any larger constant.
\hspace*{\fill}\qed
\end{pf}

The practical interpretation is that \(\varepsilon_1\) and \(\varepsilon_2\) can be tied to task/data geometry
and interpolation quality on the data support. To apply Theorem~\ref{thm:stability} to Gaussian smoothing in
the ambient space, one still needs control of off-support behavior.

We use this condition locally, either at a point of interest or on an operating region, rather than as a blanket
assumption on an arbitrary discontinuous network. In the neural-network interpretation, it means that small
perturbations of the data or parameters do not cause large output jumps on the region where the approximation is
to be trusted.

Under this condition, averaging over small input noise yields a natural infinitely smooth surrogate. The
present paper proves function-level regularity and approximation statements; optimization guarantees for the
original discrete learned parameters require additional assumptions. However, because the averaged surrogate is
\(C^\infty\), standard local smooth-optimization arguments can be applied to surrogate objectives built from
\(g\). As the input of the model, we allow the concatenation of data and parameters, so the same notation covers
perturbations in either.
Throughout, the perturbation is additive Gaussian noise, independent of both data and parameters, and we study
the expected output rather than a single noisy realization.

We denote the density of the multivariate normal distribution
\(\mathcal{N}(0,sI)\) with a zero mean and covariance \(sI\) by
\[
\phi_s(\xi)=\frac{1}{(2\pi s)^{d/2}}\exp\Bigl(-\frac{|\xi|^2}{2s}\Bigr),
\]
and the density of the standard normal distribution \(\mathcal{N}(0,I)\) by
\(\phi(\xi) := \phi_1(\xi)\).

Consider an integrable function \( f: D \to \mathbb{R}
\), \(f \in L_1(D)\), extended by zero outside \(D\). For \( s>0 \), define
\begin{equation}
\label{g(x,s)}
   g(x,s):=\mathbb{E}[f(x+\xi)] = \int\limits_{\mathbb{R}^d} f(x+\xi)\, \phi_s(\xi)\, d\xi,
\end{equation}
where \( \xi\sim \mathcal{N}(0,sI) \). Then for any multi-index \(\alpha\) there exist continuous derivatives \(
D^\alpha g(x,s)  \) given by
\begin{align*}
D^\alpha g(x, s) &:= \frac{\partial^{|\alpha|} g}{\partial x_1^{\alpha_1} \cdots \partial x_d^{\alpha_d}}(x) \\
                 &\;= (-1)^{|\alpha|} \int\limits_{\mathbb{R}^d} f(x+\xi) \, D^\alpha \phi_s(\xi) \, d\xi.
\end{align*}

\section{Main Result}

We obtain an estimate for the deviation of the averaged function from the original under a local oscillation
assumption: on the scale \(\varepsilon_1\), nearby inputs induce output changes no larger than
\(\varepsilon_2\). The result is most naturally read pointwise, or uniformly on an operating region where the
same constants \((\varepsilon_1,\varepsilon_2)\) apply.

\begin{thm}\label{thm:stability}
Let \( D \subset \mathbb{R}^d \) be some bounded open set with dimension \(d \ge 3\). Consider a function
\( f: D \to \mathbb{R} \), \(f \in L_1(D)\), extended by zero outside \(D\). Fix a point \(x\in D\) at
which \(f(x)\) is defined. Assume that there exists \(C>0\) such that \(|f(x)|\leq C\) and
\(|f(y)|\leq C\) for almost all \(y\in D\). Let constants
\(\epsilon_1 >0\), \(\epsilon_2 >0\) such that \(B(x,\epsilon_1)\subset D\) and
\[
|f(x)-f(y)|<\epsilon_2 \qquad \text{whenever } |x-y|<\epsilon_1.
\]
Assume furthermore that \(\gamma^2 := \epsilon_1^2/s \ge d-1\). Then
\begin{equation}
\label{f-g-itog}
|f(x)-g(x,s)| \le \epsilon_2 + 2\, C\,
\frac{\gamma^d \exp\!\left(-\frac{\gamma^2}{2}\right)}
{2^{\frac{d}{2}-1}\Gamma\!\left(\frac{d}{2}\right)\left(\gamma^2 - d + 2\right)} .
\end{equation}
\end{thm}

\begin{pf}
Let us estimate the modulus of the difference
\[
\begin{aligned}
|f(x)-g(x,s)| &= \left|\,\int\limits_{\mathbb{R}^d} \bigl(f(x)-f(x+\xi)\bigr) \phi_s(\xi)\,d\xi \right|\\[1mm]
&\le \int\limits_{\mathbb{R}^d} |f(x)-f(x+\xi)|\, \phi_s(\xi)\,d\xi.
\end{aligned}
\]
Split the integral into two parts:
\begin{equation}
\label{f-g}
\begin{aligned}
|f(x)-g(x,s)| &\le \int\limits_{|\xi|<\epsilon_1} |f(x)-f(x+\xi)|\, \phi_s(\xi)\,d\xi \\
&\quad+\int\limits_{|\xi|\ge \epsilon_1} |f(x)-f(x+\xi)|\, \phi_s(\xi)\,d\xi\\[1mm]
&\le \epsilon_2\,\mathbb{P}(|\xi|<\epsilon_1) + 2C\, \mathbb{P}(|\xi|\ge \epsilon_1)\\[1mm]
&\le \epsilon_2 + 2C\, \mathbb{P}(|\xi|\ge \epsilon_1).
\end{aligned}
\end{equation}
Since \( \xi\sim \mathcal{N}(0,sI) \), the equality \( \xi = \sqrt{s}\,z \) holds, where \( z\sim \mathcal{N}(0,I) \). Then
\[
\mathbb{P}(|\xi|\ge \epsilon_1)=\mathbb{P}\left(|z|\ge \gamma\right).
\]
It is known that for \( z\sim \mathcal{N}(0,I) \) the density of the norm \( |z| \) has the form
\[
p_{|z|}(r)=\frac{1}{2^{\frac{d}{2}-1}\Gamma\left(\frac{d}{2}\right)}\, r^{d-1} \exp\left(-\frac{r^2}{2}\right), \quad r\ge 0.
\]
Thus, for any \( t>0 \)
\begin{equation}
\label{Pzt}
\mathbb{P}(|z|\ge t) = \int\limits_{t}^{\infty} p_{|z|}(r)\,dr = \frac{1}{2^{\frac{d}{2}-1}\Gamma\left(\frac{d}{2}\right)} I(t),
\end{equation}
where
\begin{equation}
\label{I(t)}
I(t) := \int\limits_{t}^{\infty}  r^{d-1} \exp\left(-\frac{r^2}{2}\right) \,dr.
\end{equation}

For \(d \geq 3\) and \(t \ge \sqrt{d-1}\), using integration by parts, we have
\begin{equation}
\label{I(t)d>3}
I(t)=t^{d-2}\exp\left(-\frac{t^2}{2}\right) + (d-2)J(t),
\end{equation}
where \(J(t) := \int\limits_{t}^{\infty}  r^{d-3}\exp\left(-\frac{r^2}{2}\right) \,dr. \)

Due to the inequality \(0 < t \leq r\), we find
\[
J(t) \leq t^{-2} \int\limits_{t}^{\infty}  r^{d-1}\exp\left(-\frac{r^2}{2}\right) \,dr = \frac{I(t)}{t^{2}}.
\]
Using this estimate in equality (\ref{I(t)d>3}), we obtain
\[
I(t)(t^2 -d +2) \leq t^d\exp\left(-\frac{t^2}{2}\right).
\]
From this, we find a majorant for \(t > \sqrt{d-2}\)
\begin{equation}
\label{I(t)d-2}
I(t) \leq \frac{t^d\exp\left(-\frac{t^2}{2}\right)}{t^2 -d +2}.
\end{equation}

The integrand (\ref{I(t)}) attains its maximum at the point \(r=\sqrt{d-1}.\)
If \(t \geq \sqrt{d-1},\) then estimate (\ref{I(t)d-2}) is suitable for the integral (\ref{I(t)}).

Substituting the value \( t=\gamma \) into equality (\ref{Pzt}), using estimate (\ref{I(t)d-2}), we obtain
\[
\mathbb{P}\left(|z|\ge \gamma\right) \le \frac{1}{2^{\frac{d}{2}-1}\Gamma\left(\frac{d}{2}\right)}
\cdot \frac{\gamma^d\exp\left(-\frac{\gamma^2}{2}\right)}{\gamma^2 -d +2}.
\]
From this, taking into account inequality (\ref{f-g}), we find the desired estimate (\ref{f-g-itog}). \hspace*{\fill}\qed
\end{pf}

\begin{rem}\label{rem:operating-region}
Theorem~\ref{thm:stability} is local in the evaluation point \(x\). For quantized or piecewise-affine networks,
the natural use case is an operating region \(E\subset D\) on which the same pair
\((\epsilon_1,\epsilon_2)\) is valid for every \(x\in E\). In particular, the theorem should not be read as a
generic global property of arbitrary discontinuous networks.
\end{rem}

\begin{rem}
Bounded local oscillation is essential for approximation accuracy, not for smoothness. If \(f\) oscillates
strongly on arbitrarily small neighborhoods, then \(\epsilon_2\) in Theorem~\ref{thm:stability} cannot be made
small and Gaussian averaging need not remain faithful to the discrete model, even though \(g(\cdot,s)\) is
smooth for every \(f\in L_1(D)\).
\end{rem}

\section{Examples}

\subsection{One-dimensional smoothed activations}

We now compute Gaussian averages for the nonlinear ReLU and
\(\operatorname{sign}\) activation functions.
In this subsection, we restrict attention to the one-dimensional case \(d=1\).
We write $\phi_s(t)$ for the density of $\mathcal{N}(0,s)$ on $\mathbb{R}$ and
$\Phi_s$ (or simply $\Phi$ when $s=1$) for its cumulative distribution
function.

Let \(
\xi \) be a one-dimensional random variable distributed according to \( \mathcal{N}(0,s) \)
with \( s>0 \), and density $\phi_s$.
\begin{thm}
\label{thm:relu}
The average of the ReLU activation function
(given by \(f(x) = \max\{x,0\}\)) is
\begin{equation}
\label{gReLU}
\begin{aligned}
g(x,s)=\mathbb{E}[f(x+\xi)] &= \int\limits_{-\infty}^{+\infty}
\max\{x+\xi,0\}\, \phi_s(\xi)\,d\xi \\
&= x\,\Phi\Bigl(\frac{x}{\sqrt{s}}\Bigr) +
\sqrt{s}\,\phi\Bigl(\frac{x}{\sqrt{s}}\Bigr).
\end{aligned}
\end{equation}
\end{thm}

\begin{pf}
Note that the function \(\max\{x + \xi,0\} \) equals \( 0 \)
when \( \xi\leq -x \). Therefore, the integral (\ref{gReLU}) is
\begin{equation}
\label{gxssum}
\begin{aligned}
g(x,s)&=\int\limits_{-x}^{+\infty} (x+\xi)\, \phi_s(\xi)\,d\xi \\
&= x\int\limits_{-x}^{+\infty} \phi_s(\xi)\,d\xi
 + \int\limits_{-x}^{+\infty} \xi\,\phi_s(\xi)\,d\xi.
\end{aligned}
\end{equation}

By introducing the substitution \( z={\xi}/{\sqrt{s}} \), we find
\begin{equation}
\label{F1}
\int\limits_{-x}^{+\infty}
\phi_s(\xi)\,d\xi =\int\limits_{-x/\sqrt{s}}^{+\infty}\frac{1}{\sqrt{2\pi}}
\exp\Bigl(-\frac{z^2}{2}\Bigr)\,dz =\Phi\Bigl(\frac{x}{\sqrt{s}}\Bigr).
\end{equation}

Similarly,
\begin{equation}
\label{npn}
\begin{aligned}
\int\limits_{-x}^{+\infty} \xi\, \phi_s(\xi)\,d\xi &=\int\limits_{-x/\sqrt{s}}^{+\infty}
\frac{\sqrt{s}\,z}{\sqrt{2\pi}}\exp\Bigl(-\frac{z^2}{2}\Bigr) \,dz\\[1mm]
&= \sqrt{s}\int\limits_{-x/\sqrt{s}}^{+\infty} z\, \phi(z)\,dz.
\end{aligned}
\end{equation}
As for the standard normal distribution
\[
\int\limits_{a}^{+\infty} z\,\phi(z)\,dz = \phi(a),
\]
setting \(
a=-{x}/{\sqrt{s}} \) and taking into account the even symmetry of \(\phi\),
from (\ref{npn}) we have
\[
\int\limits_{-x}^{+\infty} \xi\, \phi_s(\xi)\,d\xi =
\sqrt{s}\,\phi\Bigl(\frac{x}{\sqrt{s}}\Bigr).
\]

Thus, from the decomposition (\ref{gxssum}), using the previous equality and the
formula (\ref{F1}), the desired representation (\ref{gReLU}) is obtained. \hspace*{\fill}\qed
\end{pf}

\begin{figure}[htbp]
\centering
\begin{tikzpicture}
  \begin{axis}[
    xlabel={\(x\)}, ylabel={\(f(x)\)}, legend pos=north west, grid=both,
    domain=-3:3, samples=160, width=6.1cm, height=5.0cm,
    scale only axis, axis on top,
  ]

    \addplot[blue, thick] {x*0.5*(1+erf(x/1.4142135623730951))};
    \addlegendentry{GELU}

    \addplot[black, thick] {0.5*ln(1+exp(2*x))};
    \addlegendentry{\(\operatorname{softplus}(2x)/2\)}

    \addplot[green!70!black, thick] {x*0.5*(1+erf(x/1.4142135623730951)) + 0.3989422804014327*exp(-0.5*x^2)};
    \addlegendentry{\(g(x, 1)\)}

    \addplot[red, thick, domain=-3:0] {0};
    \addplot[red, thick, domain=0:3] {x};
    \addlegendentry{ReLU}
    
  \end{axis}
\end{tikzpicture}\vspace{-10pt}
\caption{Smoothing of ReLU activation function}
\label{fig:relu}
\end{figure}
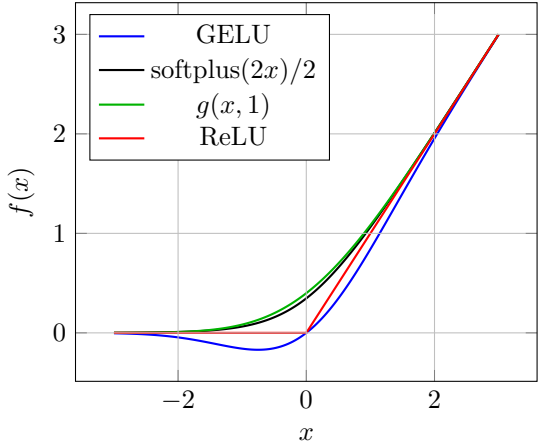

Fig.~\ref{fig:relu} compares the ReLU function, its Gaussian average at \(s = 1\), and the classical smooth
approximations given by the Gaussian error linear unit \(\operatorname{GELU}(x)=x\,\Phi(x)\)~\citep{hendrycks2016gaussian}
and
\(\operatorname{softplus}(x)=\ln(1+\exp(x))\). We plot $\operatorname{softplus}(2x)/2$ to align its slope near the origin and
its large-\(x\) behavior with ReLU.

\begin{thm}
\label{thm:sign}
Define the function sign as
\[
\operatorname{sign}(x)=
\begin{cases}
1, & x>0,\\[1mm]
-1, & x<0.
\end{cases}
\]
Then its average is
\begin{equation}
\label{gxs3}
g(x,s)=\mathbb{E}\bigl[\operatorname{sign}(x+\xi)\bigr] =2\,\Phi\Bigl(\frac{x}{\sqrt{s}}\Bigr)-1.
\end{equation}
The \(\operatorname{sign}(0)\) value is irrelevant for the expectation and can
be chosen arbitrarily.
\end{thm}

\begin{pf}
Since \(\xi\sim \mathcal{N}(0,s)\), the random variable \(x+\xi\) has a distribution \(\mathcal{N}(x,s)\). The function
\(\operatorname{sign}(x+\xi)\) takes the value \(1\) if \(x+\xi>0\), and \(-1\) if \(x+\xi<0\). Thus,
\begin{equation}
\label{gxs2}
\begin{split}
g(x,s) &= \mathbb{E}\bigl[\operatorname{sign}(x+\xi)\bigr] \\
       &= \mathbb{P}(x+\xi>0) - \mathbb{P}(x+\xi<0).
\end{split}
\end{equation}

Taking into account that \(\mathbb{P}(x+\xi>0)+\mathbb{P}(x+\xi<0)=1\), we can write
\[
g(x,s)=1-2\,\mathbb{P}(x+\xi<0).
\]
Since \(x+\xi\sim \mathcal{N}(x,s)\), the probability \( \mathbb{P}(x+\xi<0) \) equals
\[
\mathbb{P}(x+\xi<0)=\Phi\Bigl(\frac{0-x}{\sqrt{s}}\Bigr)
=\Phi\Bigl(-\frac{x}{\sqrt{s}}\Bigr).
\]
Using the symmetry of the standard normal distribution, i.e., the equality \(\Phi(-t)=1-\Phi(t)\), we obtain
\[
\mathbb{P}(x+\xi<0)=1-\Phi\Bigl(\frac{x}{\sqrt{s}}\Bigr).
\]
Substituting this into expression (\ref{gxs2}), we have
\[
g(x,s)=1-2\Bigl[1-\Phi\Bigl(\frac{x}{\sqrt{s}}\Bigr)\Bigr]
=2\,\Phi\Bigl(\frac{x}{\sqrt{s}}\Bigr)-1.
\]
This completes the proof.
\hspace*{\fill}\qed
\end{pf}

\begin{rem} The function (\ref{gxs3}), mapping the real axis \(\mathbb{R}\) to the interval \((-1,1)\), serves as a
smooth approximation of the discrete function \(\operatorname{sign}(x)\).
\end{rem}

For comparison, the logistic function \(\sigma(x)=1/(1+\exp({-x}))\) satisfies \(2\sigma(x)-1=\tanh(x/2)\), so both
logistic and \(\tanh\) give smooth surrogates of \(\operatorname{sign}(x)\) similar to \(g(x,s)\); the parameter \(s\)
in \(g(x,s)\) controls the transition steepness.

Fig.~\ref{fig:sign} compares the smoothed sign surrogate \(g(x,s)\) with the hyperbolic tangent. The choice
\(s=2/\pi\) makes the derivative of \(g(x,s)\) at the origin match that of \(\tanh(x)\), so the two curves
have comparable transition steepness.

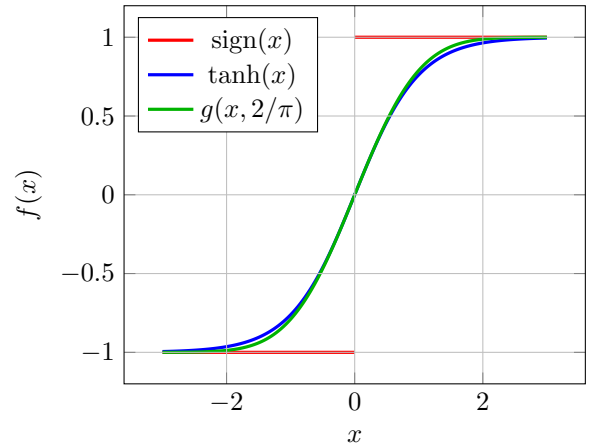
\begin{figure}[hb]
\centering
\begin{tikzpicture}
  \begin{axis}[
    xlabel={$x$},
    ylabel={$f(x)$},
    legend pos=north west,
    grid=both,
    domain=-3:3,
    samples=160,
    width=6.1cm,
    height=5.0cm,
    scale only axis, axis on top,
  ]
    \addplot[red, very thick, domain=-3:0] { -1 };
    \addplot[red, very thick, domain=0:3, forget plot] { 1 };
    \addlegendentry{\(\operatorname{sign}(x)\)}
    
    \addplot[blue, very thick] { tanh(x) };
    \addlegendentry{\(\tanh(x)\)}
    
    \addplot[green!70!black, very thick] { erf((sqrt(pi)/2)*x) };
    \addlegendentry{$g(x, 2/\pi)$}
  \end{axis}
\end{tikzpicture}\vspace{-10pt}
\caption{Smoothing of sign activation function}
\label{fig:sign}
\end{figure}

\subsection{High-dimensional binary perceptron}

We now illustrate how the bounded local oscillation condition and the Gaussian averaging mechanism interact
in a high-dimensional setting that is directly relevant to quantized policy and estimator blocks.
Threshold networks of this type appear naturally in quantized perception-and-control stacks and provide a
convenient high-dimensional test case.

Consider a two-layer binary perceptron (a building block of multilayer networks) interpreted as a quantized decision block:
\begin{equation}
\label{eq:perceptron}
f(x)=\frac{1}{\sqrt{m}}\sum_{j=1}^{m} a_j\,\sigma\!\left(\frac{w_j^\top x}{\sqrt{n}}-\theta_j\right),
\end{equation}
where \(x\in\mathbb{R}^n\) is the state or measurement vector, \(m\) is the hidden width,
\(\sigma\) is a sign-type or ReLU-type primitive, \(a_j\in\mathbb{R}\) are second-layer coefficients, $\theta_j\in\mathbb{R}$ is the bias (threshold) of the $j$-th hidden neuron,
and \(w_j^\top=(w_{jk})_{k=1}^n\), where \(w_{jk}\in\{-1,+1\}\) are quantized first-layer weights.

\begin{thm}\label{thm:layer-aggregation}
Fix a hidden unit \(j\) and a quantization step \(\Delta>0\). For each input dimension \(n\), suppose that:
\begin{enumerate}
\item the layer weights are deterministic and binary,
\(w_{jk}^{(n)}\in\{-1,+1\}\), and the threshold \(\theta_j^{(n)}\in\mathbb{R}\) is deterministic;
\item the input vector \(x^{(n)}=(x_1^{(n)},\ldots,x_n^{(n)})^{\top}\in\mathbb{R}^n\) is deterministic;
\item the componentwise quantization residuals are modeled via the additive-noise surrogate
\[
Q_\Delta(x_k^{(n)})=x_k^{(n)}+\delta_{k,n},
\]
where \(\delta_{1,n},\ldots,\delta_{n,n}\) are i.i.d.\ with the centered uniform law on
\([-\Delta/2,\Delta/2]\), consistent with the required assumptions in the classical quantization model \citep{lipshitz1992quantization}.
\end{enumerate}
Define
\[
\begin{aligned}
\bar z_{j,n} &=
\frac{1}{\sqrt{n}}\sum_{k=1}^{n} w_{jk}^{(n)}x_k^{(n)}-\theta_j^{(n)}, \\
\delta z_{j,n} &=
\frac{1}{\sqrt{n}}\sum_{k=1}^{n} w_{jk}^{(n)}\delta_{k,n}, \\
\hat z_{j,n} &= \bar z_{j,n}+\delta z_{j,n}.
\end{aligned}
\]
Then the preactivation perturbation satisfies
\[
\mathbb{E}[\delta z_{j,n}]=0,\qquad
\operatorname{Var}(\delta z_{j,n})=\frac{\Delta^2}{12},
\]
and
\[
\delta z_{j,n}\xrightarrow{d}\mathcal{N}\!\left(0,\frac{\Delta^2}{12}\right)
\qquad\text{as } n\to\infty.
\]
Consequently, if \(\bar z_{j,n}\to z_j\) along the same sequence, then
\[
\hat z_{j,n}\xrightarrow{d}\mathcal{N}\!\left(z_j,\frac{\Delta^2}{12}\right).
\]
\end{thm}

\begin{pf}
Set \(Y_{k,n}=w_{jk}^{(n)}\delta_{k,n}\). Since \(w_{jk}^{(n)}\in\{-1,+1\}\) is deterministic and the
uniform law on \([-\Delta/2,\Delta/2]\) is symmetric, each \(Y_{k,n}\) has the same law as \(\delta_{k,n}\).
Hence, for each \(n\), the variables \(Y_{1,n},\ldots,Y_{n,n}\) are i.i.d.\ with
\[
\mathbb{E}[Y_{k,n}]=0,\qquad
\operatorname{Var}(Y_{k,n})=\frac{\Delta^2}{12}.
\]
Therefore
\[
\mathbb{E}[\delta z_{j,n}]
=\frac{1}{\sqrt{n}}\sum_{k=1}^{n}\mathbb{E}[Y_{k,n}]=0
\]
and, by independence,
\[
\operatorname{Var}(\delta z_{j,n})
=\frac{1}{n}\sum_{k=1}^{n}\operatorname{Var}(Y_{k,n})
=\frac{\Delta^2}{12}.
\]
Because the \(Y_{k,n}\) are i.i.d.\ with finite variance, the classical central limit theorem gives
\[
\frac{1}{\sqrt{n}}\sum_{k=1}^{n}Y_{k,n}
\xrightarrow{d}\mathcal{N}\!\left(0,\frac{\Delta^2}{12}\right),
\]
which is the stated convergence of \(\delta z_{j,n}\). If \(\bar z_{j,n}\to z_j\), Slutsky's theorem,
stated for example as Slutsky's lemma by \citet[Lemma~2.8]{vanDerVaart-1998-AsymptoticStatistics}, gives
the corresponding convergence of \(\hat z_{j,n}=\bar z_{j,n}+\delta z_{j,n}\).
\hspace*{\fill}\qed
\end{pf}

Theorem~\ref{thm:layer-aggregation} makes the perturbation endogenous: it comes from componentwise
input quantization. For any actual fixed input the residuals are deterministic; the theorem states an
explicit i.i.d.\ additive-noise surrogate under which the layer preactivation satisfies a standard CLT.
Thus the Gaussian object appears only after aggregation in the linear preactivation, not at the level of the
scalar quantizer itself.

\begin{rem}\label{rem:multilayer-independence}
In a multilayer network, exact independence of internal perturbations is generally lost after the first
nonlinearity without considering the geometry of the latent space. Accordingly, the Gaussian smoothing picture 
should be read as a coarse-grained empirical surrogate model, not as a literal layer-wise central-limit identity; 
\citep{Salishev-2026-SPSAView} provides empirical support for this viewpoint in both training and inference.
On the training side, smooth surrogate gradients give stable low-bit optimization behavior. On the inference
side, for smooth image-to-image super-resolution, the reported W4A4 RFDN experiments keep overlapping final PSNR
confidence intervals deterministic rounding-noise variants. This suggests that, on the
tested data distribution, harmful quantization events are rare or small enough to be absorbed by the
reconstruction task.
\end{rem}

\begin{cor}\label{cor:smoothed-activation}
Under the conditions of Theorem~\ref{thm:layer-aggregation}, if \(\sigma\) is the sign function,
the quantized layer activation \(\sigma(\hat z_{j,n})\) is replaced in the coarse-grained surrogate by the
smoothed activation
\[
\tilde\sigma(\bar z_{j,n}) = 2\,\Phi\!\left(\frac{\bar z_{j,n}}{\sqrt{s}}\right)-1,
\]
where \(s\) is the variance of the Gaussian approximation to \(\delta z_{j,n}\); in the i.i.d.\ uniform
sampling model above, \(s=\Delta^2/12\). For the ReLU function, the corresponding surrogate is
\(\bar z_{j,n}\,\Phi(\bar z_{j,n}/\sqrt{s})+\sqrt{s}\,\phi(\bar z_{j,n}/\sqrt{s})\). This constitutes a
coarse-grained surrogate architecture, not an exact equality for the full multilayer network.
\end{cor}

\begin{rem}\label{rem:training-averaging}
The same limiting-envelope viewpoint also appears in training: batch losses average over data samples, and the
quantized version of \eqref{eq:perceptron} averages many small componentwise residuals inside each
preactivation. The training-side analysis of noisy STE~\citep{Bengio-2013-STE} and noise-scale updates is developed in
\citep{Salishev-2026-SPSAView}. Applying continuous gradient descent to binary weights is implemented 
in GDNSQ~\citep{Salishev-2025-GDNSQ} and is not a point of this paper.
\end{rem}

\begin{rem}\label{rem:control-transfer}
The payoff for control is local surrogate use rather than exact equivalence. In inference, layer-wise
aggregation yields the Gaussian envelope used in the surrogate activations; in training, batch averaging yields
the analogous averaged object. On operating regions where the perceptron has
\((\epsilon_1,\epsilon_2)\)-bounded local oscillation, Theorem~\ref{thm:stability} then controls the mismatch
between that Gaussian-averaged proxy and the discrete model.
\end{rem}

\section{Conclusion}
\label{sec:conclusion}

Gaussian convolution supplies a \(C^\infty\) surrogate for discontinuous quantized models, while bounded local
oscillation determines when that surrogate remains accurate and yields the local error bound of
Theorem~\ref{thm:stability}. In the high-dimensional perceptron example, Theorem~\ref{thm:layer-aggregation}
shows that, under an explicit i.i.d.\ additive-noise quantization surrogate, coordinatewise residuals aggregate
into a Gaussian envelope for a layer preactivation. This produces the smoothed activations studied here and
supplies the same \(C^\infty\) surrogate for inference-side analysis and training-side smooth surrogate
gradients. Theorem~\ref{thm:layer-aggregation} has only layer-wise reading 
without analyzing latent space geometry and thus has mostly empirical support for multilayer networks.

Theorem~\ref{thm:stability} itself is depth-agnostic and applies to networks of arbitrary depth wherever bounded
local oscillation holds. The resulting surrogate should be understood as a smooth probabilistic proxy whose
local mismatch is controlled on trusted data-supported regions, while excursions outside those regions are
rare-event disturbances to be handled by control-specific robustness or chance-constrained analysis. Future work
should include a theoretical analysis of internal latent geometry beyond a single layer, as well as broader
catalogs of smoothed primitives and control-oriented tools built on these surrogates.


\section*{DECLARATION OF GENERATIVE AI AND AI-ASSISTED TECHNOLOGIES IN THE WRITING PROCESS}
During the preparation of this work, the authors used OpenAI ChatGPT in order to assist with 
language editing, formatting, and drafting of the abstract and keywords. After using this tool, 
the authors reviewed and edited the content as needed and take full responsibility for the content of the publication.

\bibliography{main}             

\end{document}